\documentclass{article}

\usepackage{microtype}
\usepackage{graphicx}
\usepackage{subfigure}
\usepackage{booktabs} 
\usepackage{makecell}
\usepackage{multirow} 

\usepackage{bbm}

\usepackage{xcolor}
\usepackage{mdframed}
\usepackage{stfloats}
\usepackage[table]{xcolor}

\usepackage{hyperref}

\usepackage[accepted]{icml2025}

\usepackage{amsmath}
\usepackage{amssymb}
\usepackage{mathtools}
\usepackage{amsthm}
\usepackage{enumitem}

\renewcommand{\algorithmicrequire}{\textbf{Input:}}
\renewcommand{\algorithmicensure}{\textbf{Output:}}

\usepackage[capitalize,noabbrev]{cleveref}

\theoremstyle{plain}
\newtheorem{theorem}{Theorem}[section]

\newtheorem{lemma}[theorem]{Lemma}
\newtheorem{corollary}[theorem]{Corollary}
\theoremstyle{definition}
\newtheorem{definition}[theorem]{Definition}
\newtheorem{assumption}[theorem]{Assumption}
\theoremstyle{remark}
\newtheorem{remark}[theorem]{Remark}

\usepackage[textsize=tiny]{todonotes}

\icmltitlerunning{SAGE: Accelerating Vision-Language Models via Entropy-Guided Adaptive Speculative Decoding}

\begin{document}

\twocolumn[
\icmltitle{SAGE: Accelerating Vision-Language Models via Entropy-Guided Adaptive Speculative Decoding }



\icmlsetsymbol{equal}{*}

\begin{icmlauthorlist}
\icmlauthor{Yujia Tong}{yyy}
\icmlauthor{Tian Zhang}{yyy}
\icmlauthor{Yunyang Wan}{yyy}
\icmlauthor{Kaiwei Lin}{yyy}
\icmlauthor{Jingling Yuan}{yyy}
\icmlauthor{Chuang Hu}{sch}

\end{icmlauthorlist}

\icmlaffiliation{yyy}{School of Computer Science and Artificial Intelligence, Wuhan University of Technology, Hubei 430070, China. }

\icmlaffiliation{sch}{School of Computer Science, Wuhan University, Hubei 430072, China}

\icmlcorrespondingauthor{Yujia Tong}{tyjjjj@whut.edu.cn}

\icmlkeywords{Machine Learning, ICML}

\vskip 0.3in
]

\printAffiliationsAndNotice{}  

\begin{abstract}
Speculative decoding has emerged as a promising approach to accelerate inference in vision-language models (VLMs) by enabling parallel verification of multiple draft tokens. However, existing methods rely on static tree structures that remain fixed throughout the decoding process, failing to adapt to the varying prediction difficulty across generation steps. This leads to suboptimal acceptance lengths and limited speedup.  In this paper, we propose \textbf{SAGE}, a novel framework that dynamically adjusts the speculation tree structure based on real-time prediction uncertainty. Our key insight is that output entropy serves as a natural confidence indicator with strong temporal correlation across decoding steps. SAGE constructs deeper-narrower trees for high-confidence predictions to maximize speculation depth, and shallower-wider trees for uncertain predictions to diversify exploration. SAGE improves acceptance lengths and achieves faster acceleration compared to static tree baselines. Experiments on multiple benchmarks demonstrate the effectiveness of SAGE: without any loss in output quality, it delivers up to $3.36\times$ decoding speedup for LLaVA-OneVision-72B and $3.18\times$ for Qwen2.5-VL-72B.
\end{abstract}

\section{Introduction}
Although vision-language models (VLMs)~\cite{li2024llava,zhang2024vision} exhibit outstanding performance in multimodal tasks, their enormous parameter scale often leads to a sharp increase in computational overhead during the inference phase, including higher memory demands and longer response times. Due to the need to process both visual and textual information simultaneously, VLMs typically require more computational resources than single-modality models~\cite{Vasu_2025_CVPR}, making the latency bottleneck in auto-regressive generation particularly prominent.

A promising approach to alleviating this bottleneck is Speculative Decoding ~\cite{leviathan2023fast,sun2023spectr}. The core idea of this approach lies in breaking the limitation of traditional auto-regressive models, which can only generate tokens one by one. Its working mechanism typically relies on an efficient small draft model and a powerful original target model. During the generation process, the small draft model rapidly and continuously predicts multiple candidate tokens, forming a draft sequence. Subsequently, the original target model verifies the entire draft sequence in parallel in a single step, retaining only the correct prefix. By leveraging the parallel computing capability of the target model, speculative decoding achieves a significant acceleration in inference while strictly ensuring that the output remains completely consistent with that of the original model~\cite{xia2023speculative,kim2023speculative}.

However, most existing research on speculative decoding ~\cite{leviathan2023fast,sun2023spectr} has primarily focused on large language models (LLMs), with relatively limited exploration in the domain of vision-language models. Recently, studies such as SpecVLM~\cite{ji2025specvlm} have begun to address this gap by proposing speculative decoding frameworks specifically tailored for VLMs. SpecVLM introduces a training-free speculative decoding approach designed for Video LLMs, which combines token pruning with tree-based draft generation to accelerate inference. 

Despite these advancements, a critical limitation persists in current speculative decoding approaches: the tree structure governing draft generation is typically defined statically prior to inference and remains fixed throughout the decoding process. This static configuration fails to account for the inherent variability in prediction difficulty across different generation steps. The model's prediction confidence varies substantially across tokens. Deterministic elements such as domain-specific terminology produce concentrated probability distributions with low entropy and high confidence. Open-ended content and creative expressions, by contrast, yield dispersed probability distributions with high entropy and low confidence~\cite{Holtzman2020The,li2024eagle2}. A fixed tree structure cannot adapt to these varying conditions: when entropy is low, a narrow and shallow tree wastes the opportunity to speculate further ahead; conversely, when entropy is high, an overly deep tree leads to wasted computation on branches that are unlikely to be accepted. This mismatch between static tree configurations and dynamic prediction entropy results in suboptimal acceptance lengths and consequently limits the achievable speedup ratio.

To address this challenge, we propose SAGE, a novel framework that dynamically adjusts the tree structure based on the model's real-time prediction uncertainty. Our key insight is that the entropy of the output probability distribution serves as a natural and computationally inexpensive indicator of prediction confidence. Specifically, at each decoding step, we compute the normalized entropy of the draft model's output distribution and derive a confidence score inversely proportional to this entropy. This confidence score then guides the adaptive construction of the speculation tree: when confidence is high (low entropy), we construct deeper but narrower trees to capitalize on the high acceptance probability and speculate further into the future; when confidence is low (high entropy), we construct shallower but wider trees to explore more candidate branches while avoiding wasted computation on deep paths prone to rejection. Furthermore, we exploit the temporal correlation of entropy across consecutive decoding steps—empirically, adjacent tokens tend to exhibit similar uncertainty levels—allowing us to use the current step's entropy to effectively optimize the next step's tree structure with negligible computational overhead. Our approach requires no additional training and can be seamlessly integrated into existing speculative decoding frameworks for VLMs, achieving improved acceptance lengths and superior inference acceleration while maintaining output equivalence with the original model.

Our key contributions are as follows:

\begin{itemize}[leftmargin=*]

\item  We reveal that the entropy of draft model outputs serves as a natural confidence indicator, which motivates our entropy-guided adaptive tree construction strategy.

\item  We propose SAGE, which dynamically constructs deeper-narrower trees for high-confidence predictions and shallower-wider trees for uncertain ones based on real-time entropy estimation.

\item We provide theoretical analysis establishing the relationship between prediction entropy and token acceptance probability, and derive optimal tree depth/width configurations that justify our adaptive strategy.

\item Extensive experiments on multiple VLM benchmarks demonstrate that SAGE achieves higher acceptance lengths and superior speedup compared to static tree baselines while maintaining output equivalence.

\end{itemize}

\section{Preliminaries}

\subsection{Vision-Language Models}

Vision-language models (VLMs) adopt an encoder-decoder architecture. Given an input image $I$, the visual encoder extracts visual tokens $\mathbf{Z} = \{z_1, z_2, \ldots, z_m\}$, which are projected into the language model's embedding space. The language model generates output tokens auto-regressively:
\begin{equation}
P(y_t \mid \mathbf{Z}, \mathbf{y}_{<t}) = \mathrm{softmax}(\mathbf{W}_o \mathbf{h}_t),
\end{equation}
where $\mathbf{h}_t$ denotes the hidden state at step $t$ and $\mathbf{W}_o$ is the output projection matrix. The complete response $\mathbf{Y} = \{y_1, \ldots, y_n\}$ is generated as:
\begin{equation}
P(\mathbf{Y} \mid \mathbf{Z}) = \prod_{t=1}^{n} P(y_t \mid \mathbf{Z}, \mathbf{y}_{<t}).
\end{equation}

Due to this auto-regressive nature, the decoding process suffers from significant latency, as each token must be generated sequentially with \( n \) forward passes required to produce a response of length \( n \). This sequential dependency prevents parallelization during inference, making the decoding phase the primary computational bottleneck in VLMs, especially when generating long-form responses.

\subsection{Speculative Decoding}

Speculative decoding accelerates auto-regressive generation by using a smaller draft model $\mathcal{M}_d$ to speculate $\gamma$ candidate tokens, which are verified in parallel by the target model $\mathcal{M}_t$. Given candidates $\hat{y}_{t:t+\gamma}$, the target model computes:
\begin{equation}
P_{\mathcal{M}_t}(y_i \mid \mathbf{Z}, \mathbf{y}_{<t}, \hat{y}_{t:i-1}), \quad \forall i \in \{t, \ldots, t+\gamma\}.
\end{equation}
The acceptance length $\tau$ is the longest prefix matching the target's predictions, advancing $\tau + 1$ tokens per iteration.

\textbf{Tree-based Drafting.} Instead of a single sequence, the draft model constructs a candidate tree $\mathcal{T}$. At depth $l$, top-$k$ tokens are selected:
\begin{equation}
\mathrm{Top}_k\big(P_{\mathcal{M}_d}(y \mid \mathbf{Z}, \mathbf{y}_{<t}, \mathbf{p}_{<l})\big),
\end{equation}
where $\mathbf{p} = [p_1, \ldots, p_d]$ denotes the path indices. A tree attention mask $\mathbf{M} \in \mathbb{R}^{|\mathcal{T}| \times |\mathcal{T}|}$ ensures each token only attends to its ancestors. The optimal path is:
\begin{equation}
\mathbf{p}^* = \arg\max_{\mathbf{p} \in \mathcal{T}} \sum_{l=1}^{|\mathbf{p}|} \mathbbm{1}\big[\hat{y}_{\mathbf{p}_{\leq l}} = \arg\max P_{\mathcal{M}_t}(y \mid \cdot)\big].
\end{equation}
The expected speedup ratio is:
\begin{equation}
S \approx \frac{\mathbb{E}[\tau] + 1}{c_d \cdot |\mathcal{T}| + c_t},
\end{equation}
where $c_d$ and $c_t$ denote the computational costs of the draft and target models, respectively.

\begin{figure}[t]
\centering
\includegraphics[width=0.492\textwidth]{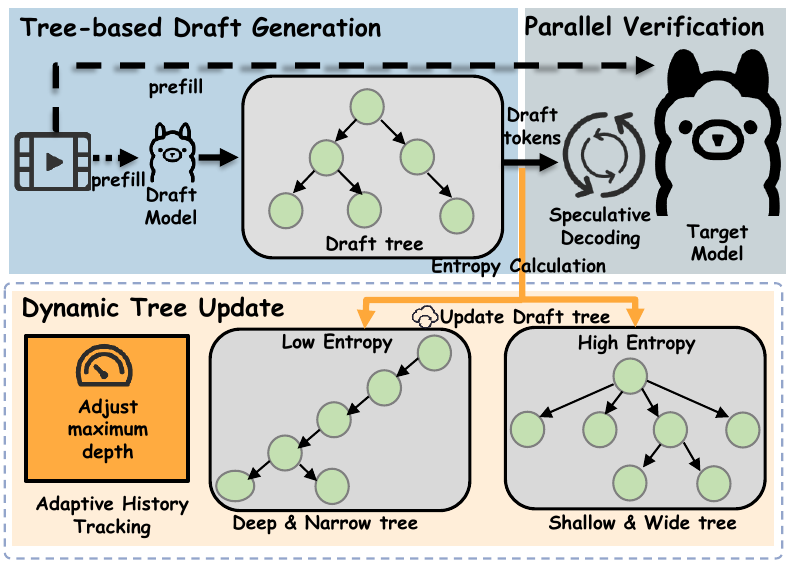}
\vspace{-15pt}
\caption{Overview of SAGE. The framework consists of three phases: 
(1) tree-based draft generation, (2) parallel verification by the target 
model, and (3) entropy-guided dynamic tree update that computes a 
confidence score from output entropy and adapts tree depth/width accordingly.}
\label{fig1}
\end{figure}

\section{Methodology}
In this section, we present our SAGE framework. We first introduce how to estimate prediction confidence using output entropy (\S~\ref{S3.1}), then describe how to adaptively construct tree structures based on this confidence (\S~\ref{S3.2}), and finally present the overall inference pipeline (\S~\ref{S3.3}).

\subsection{Entropy-Based Confidence Estimation}
\label{S3.1}
The key insight of our approach is that the entropy of the output probability distribution serves as a natural indicator of prediction uncertainty. When the model is confident about its prediction, the probability mass concentrates on a few tokens, resulting in low entropy; conversely, high uncertainty leads to a more uniform distribution with high entropy~\cite{Holtzman2020The,li2024eagle2}. 

Given the output logits $\mathbf{o}_t \in \mathbb{R}^{|V|}$ from the draft model at decoding step $t$, we extract the top-$k$ probabilities and renormalize them:
\begin{equation}
\tilde{P}_i = \frac{P_i}{\sum_{j=1}^{k} P_j}, \quad \forall i \in \mathrm{Top}_k.
\label{eq:renorm}
\end{equation}
We then compute the Shannon entropy of this renormalized distribution:
\begin{equation}
H(\tilde{\mathbf{P}}) = -\sum_{i=1}^{k} \tilde{P}_i \log \tilde{P}_i.
\label{eq:entropy}
\end{equation}
The confidence score $\alpha \in [0, 1]$ is defined as the complement of the normalized entropy:
\begin{equation}
\alpha = 1 - \frac{H(\tilde{\mathbf{P}})}{\log k}.
\label{eq:confidence}
\end{equation}

When the model is highly confident , the probability distribution is peaked on a single token; when completely uncertain, the distribution approaches uniform. This confidence score directly guides our adaptive tree construction strategy described in the next subsection.

\subsection{Adaptive Tree Structure Generation}
\label{S3.2}
Based on the confidence score $\alpha$, we dynamically adjust both the depth and width of the speculation tree. Our guiding principle is twofold. For \textbf{high confidence} scenarios ($\alpha \to 1$), we construct deeper but narrower trees, since the draft model's predictions are likely to be accepted and we can speculate further into the future while reducing unnecessary exploration of alternative branches. For \textbf{low confidence} scenarios ($\alpha \to 0$), we construct shallower but wider trees, as it is more beneficial to explore diverse candidates at shallow depths rather than committing to deep paths that are unlikely to be accepted.

\paragraph{Adaptive Depth.} Given the confidence score $\alpha$, we compute the target tree depth as a linear interpolation between minimum and maximum depths:
\begin{equation}
D(\alpha) = D_{\min} + \alpha \cdot (D_{\max} - D_{\min}),
\label{eq:depth}
\end{equation}
where $D_{\min}$ and $D_{\max}$ are hyperparameters controlling the depth range.

\paragraph{Adaptive Width.} Conversely, the tree width is computed as:
\begin{equation}
W(\alpha) = W_{\min} + (1 - \alpha) \cdot (W_{\max} - W_{\min}),
\label{eq:width}
\end{equation}
where $W_{\min}$ and $W_{\max}$ define the width range. This ensures that uncertain predictions trigger wider exploration at shallow levels.

\paragraph{Hierarchical Width Decay.} Beyond the first level, we apply a depth-dependent decay to the width at each subsequent level $l$:
\begin{equation}
W_l = W(\alpha) \cdot \frac{1}{l} \cdot (0.5 + P_{\mathrm{parent}}),
\label{eq:decay}
\end{equation}
where $P_{\mathrm{parent}}$ is the probability of the parent node. This design ensures that deeper levels have progressively fewer branches, avoiding exponential growth, and that higher-probability branches receive more exploration resources than lower-probability ones.

\paragraph{Dynamic Tree Construction.} Given the adaptive parameters, we construct the tree through a recursive expansion procedure starting from the root node. At each node, we first compute the local width $W_l$ based on the current depth and parent probability, then select the top-$W_l$ candidates from the draft model's output distribution. For each candidate whose probability exceeds a depth-adaptive threshold $\theta_l = 0.1 \cdot l / D(\alpha)$, we recursively expand its subtree. The expansion terminates when reaching either the maximum depth $D(\alpha)$ or the maximum node count $N_{\max}$.

The resulting tree $\mathcal{T}_\alpha$ adapts its structure based on prediction confidence, containing more nodes along high-probability paths and fewer along uncertain ones. Formally, a tree path $\mathbf{p} = [p_1, p_2, \ldots, p_d]$ is included in $\mathcal{T}_\alpha$ if and only if for all levels $l \leq d$, the path index $p_l < W_l$ and the cumulative probability $P(\hat{y}_{\mathbf{p}_{\leq l}}) > \theta_l$.

\begin{figure}[t]
\centering
\includegraphics[width=0.472\textwidth]{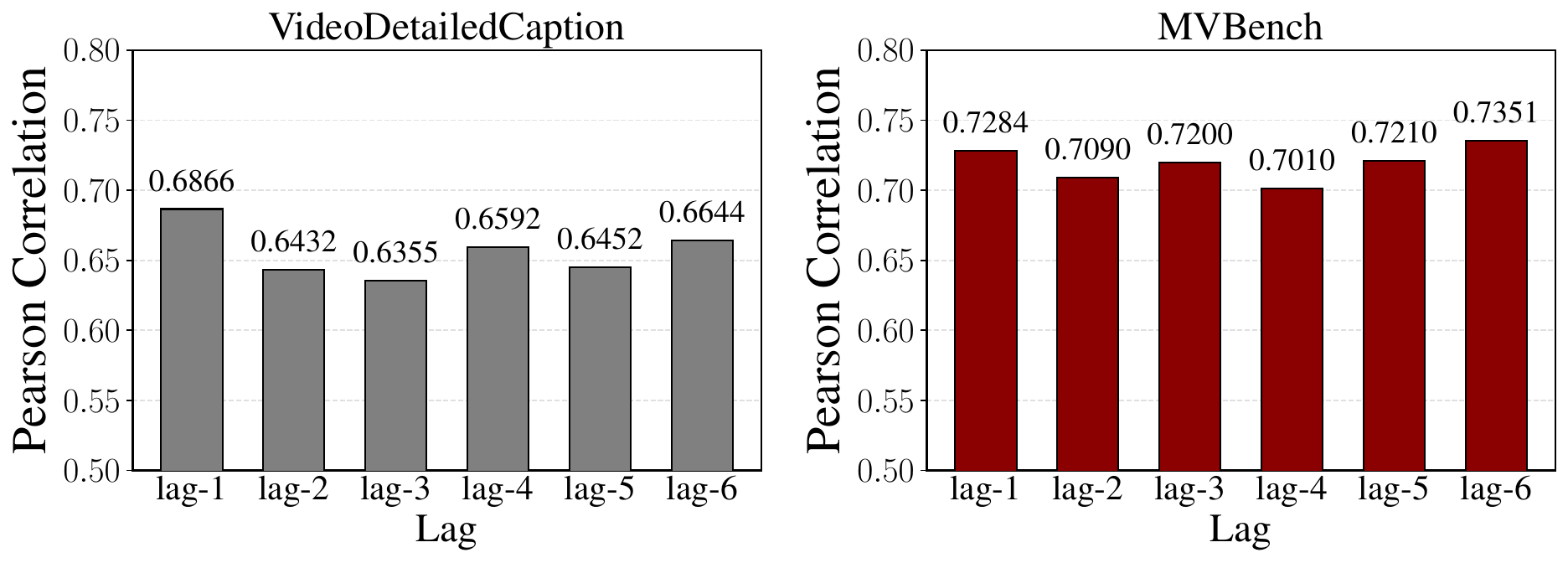}
\caption{Autocorrelation analysis of prediction entropy sequences on Qwen2.5-VL 7B. Left: VideoDetailedCaption dataset. Right: MVBench dataset. Lag-$k$ denotes the correlation between entropy values separated by $k$ decoding steps.}
\label{fig:entropy_acf}
\end{figure}

\paragraph{Temporal Correlation Exploitation.} An important empirical observation for long-form text generation in video understanding tasks is that prediction confidence exhibits strong temporal correlation—adjacent tokens in a sequence tend to have similar uncertainty levels. As shown in Figure~\ref{fig:entropy_acf}, the Pearson correlation coefficients between entropy values at different lags remain consistently high across both VideoDetailedCaption and MVBench datasets (all $p < 0.05$), demonstrating strong temporal consistency of entropy. This allows us to use the entropy computed at step $t$ to optimize the tree structure for step $t+1$ with negligible additional overhead.

\subsection{Overall Inference Pipeline}
\label{S3.3}

Algorithm~\ref{alg:SAGE} presents the complete SAGE inference procedure. The process consists of three phases: initialization, iterative speculation-verification, and dynamic tree update.

\begin{algorithm}[t]
\small
\caption{The Overall Inference Pipeline of SAGE}
\label{alg:SAGE}
\begin{algorithmic}[1]
\REQUIRE Target model $\mathcal{M}_t$, draft model $\mathcal{M}_d$, input $(\mathbf{Z}, \mathbf{y}_{<1})$, max generation tokens $T$, Dynamic  tree parameters: $D_{\min}, D_{\max}, W_{\min}, W_{\max}, k, N_{\max}$
\STATE Initialize KV caches for $\mathcal{M}_t$ and $\mathcal{M}_d$
\STATE Initialize tree structure $\mathcal{T}$ 
\STATE Initialize confidence $t \gets 1$, $\alpha \gets 0.5$ 
\WHILE{$t < T$ and not \textsc{EOS}}
    \STATE \textcolor[RGB]{34,149,34}{// \textbf{Phase 1: Tree-based Draft Generation}}
    \STATE $\{\hat{y}_\mathbf{p}\}_{\mathbf{p} \in \mathcal{T}}, \mathbf{P}_{\mathrm{out}} \gets \mathcal{M}_d(y_{t-1}, \mathcal{T})$
    \STATE \textcolor[RGB]{34,149,34}{// \textbf{Phase 2: Parallel Verification}}
    \STATE $\mathbf{M} \gets$ \textsc{TreeAttnMask}($\mathcal{T}$)
    \STATE $\{P_{\mathcal{M}_t}(\cdot \mid \mathbf{p})\}_{\mathbf{p} \in \mathcal{T}} \gets \mathcal{M}_t(\{\hat{y}_\mathbf{p}\}, \mathbf{M})$
    \STATE $\mathbf{p}^*, \tau \gets$ \textsc{EvaluatePosterior}($\mathcal{T}, P_{\mathcal{M}_t}$)
    \STATE $y_{t:t+\tau} \gets \hat{y}_{\mathbf{p}^*_{1:\tau+1}}$
    \STATE Update KV caches with accepted tokens
    \STATE $t \gets t + \tau + 1$
    \STATE \textcolor[RGB]{34,149,34}{// \textbf{Phase 3: Dynamic Tree Update}}
    \STATE Compute Confidence $\alpha$ using $\mathbf{P}_{\mathrm{out}}$ by Eq.(\ref{eq:confidence})
    \STATE Update $D \gets D_{\min} + \alpha \cdot (D_{\max} - D_{\min})$ by Eq.(\ref{eq:depth})
    \STATE Update $W \gets W_{\min} + (1-\alpha) \cdot (W_{\max} - W_{\min})$ by Eq.(\ref{eq:width})
    \STATE Update $\mathcal{T}$ using ($\mathbf{P}_{\mathrm{out}}, D, W, N_{\max}$)
\ENDWHILE
\STATE \textbf{Return} Generated sequence $\mathbf{y}_{1:t}$
\end{algorithmic}
\end{algorithm}

\paragraph{Phase 1: Tree-based Draft Generation.} Given the current tree structure $\mathcal{T}$, the draft model generates candidate tokens for all paths in the tree. The tree attention mask ensures proper causal attention within each path, allowing each token to only attend to its ancestors in the tree. We retain the output probability distribution $\mathbf{P}_{\mathrm{out}}$ for entropy computation in the subsequent dynamic tree update phase.

\paragraph{Phase 2: Parallel Verification.} The target model verifies all candidate paths in a single forward pass using the tree attention mask. For each path in the tree, we compare the draft tokens against the target model's greedy predictions at each position. The best path is selected as the one achieving the longest accepted prefix, where all draft tokens along the path match the corresponding greedy outputs from the target model. The accepted tokens are then appended to the generated sequence, and the key-value caches of both models are updated accordingly.

\paragraph{Phase 3: Dynamic Tree Update.} Using the retained probability distribution $\mathbf{P}_{\mathrm{out}}$ from the draft generation phase, we compute the confidence score based on the normalized entropy as described in Eq.(\ref{eq:confidence}). This confidence score then determines the depth and width parameters for constructing the next tree structure through Eq.(\ref{eq:depth}) and Eq.(\ref{eq:width}), enabling the speculation strategy to adapt to the current prediction uncertainty before the next iteration begins.

\paragraph{Adaptive History Tracking.} Beyond the entropy-based per-step adaptation, we introduce a feedback mechanism that adjusts the tree configuration based on historical acceptance performance. Specifically, we maintain a sliding window of recent acceptance lengths and compute their moving average. When the recent average acceptance length falls below a lower threshold, it indicates that the draft model's predictions are frequently rejected by the target model, and we reduce the maximum depth to avoid wasting computation on deep speculation paths that are unlikely to be accepted. Conversely, when the recent average acceptance length exceeds an upper threshold, it suggests strong alignment between draft and target models, and we increase the maximum depth to capitalize on the high acceptance rate and speculate further into the future. When the average falls between these two thresholds, we maintain the current depth unchanged, providing hysteresis to prevent frequent oscillation between configurations. This history-based adaptation complements the entropy-guided strategy: while entropy captures instantaneous prediction uncertainty, the acceptance history reflects cumulative draft-target alignment quality over recent steps, enabling more robust adaptation to varying generation difficulty throughout the decoding.

\paragraph{Complexity Analysis.} The entropy computation adds $\mathcal{O}(k)$ operations per decoding step, which is negligible compared to the $\mathcal{O}(|V| \cdot d)$ complexity of the forward pass, where $|V|$ denotes the vocabulary size and $d$ is the hidden dimension. The tree generation requires $\mathcal{O}(N_{\max})$ operations in the worst case. Since $k, N_{\max} \ll |V| \cdot d$, our method introduces minimal overhead while enabling significant speedup through improved acceptance lengths. We provide a detailed time breakdown analysis in Appendix~\ref{sec:time_analysis}.

\section{Theoretical Analysis}

In this section, we provide theoretical justifications for our entropy-guided adaptive tree construction. We establish the relationship between prediction entropy and token acceptance probability (\S~\ref{S4.1}), then analyze the optimal tree configuration (\S~\ref{S4.2}). Due to space constraints, detailed proofs are deferred to Appendix~\ref{app:proofs}.

\subsection{Entropy and Acceptance Probability}
\label{S4.1}

We formalize the connection between output entropy and the probability that a draft token is accepted by the target model.

\begin{definition}[Acceptance Event]
Let $P_d(\cdot|c)$ and $P_t(\cdot|c)$ denote the output distributions of the draft and target models conditioned on context $c$, respectively. Under greedy decoding, a draft token $\hat{y} = \arg\max_y P_d(y|c)$ is accepted if and only if $\hat{y} = \arg\max_y P_t(y|c)$.
\end{definition}

\begin{assumption}[Bounded Distribution Divergence]
\label{ass:alignment}
The draft and target models have bounded total variation distance:
$D_{TV}(P_d(\cdot|c), P_t(\cdot|c)) \leq \epsilon$ for some small $\epsilon \geq 0$.
\end{assumption}

\begin{lemma}[Confidence-Probability Relationship]
\label{lem:conf_prob}
Let $\tilde{P}_d$ denote the renormalized top-$k$ distribution with probabilities $p_1 \geq p_2 \geq \cdots \geq p_k$, and let $\alpha = 1 - H(\tilde{P}_d)/\log k$ be the confidence score. Then:
\begin{equation}
p_1 \geq \frac{1}{k} + \alpha \cdot \frac{k-1}{k}.
\label{eq:p1_lower}
\end{equation}
\end{lemma}

\begin{theorem}[Acceptance Probability Lower Bound]
\label{thm:accept_prob}
Under Assumption~\ref{ass:alignment}, acceptance is guaranteed when:
\begin{equation}
\alpha > \frac{k - 2 + 4\epsilon k}{2(k-1)}.
\label{eq:alpha_threshold}
\end{equation}
For $k=10$ and $\epsilon=0.05$, this threshold is approximately $0.56$.
\end{theorem}

This theorem establishes that high confidence (low entropy) provides a sufficient condition for token acceptance, justifying our use of entropy as a signal for tree construction.

\subsection{Optimal Tree Configuration}
\label{S4.2}

We analyze the optimal tree depth and width to maximize speedup. We adopt a simplified model to derive \textbf{qualitative insights that motivate our adaptive strategy}.

\begin{assumption}[Acceptance Probability Model]
\label{ass:accept_model}
The acceptance probability at depth $l$ is $p_l = p \cdot \gamma^{l-1}$ for base probability $p \in (0,1]$ and decay factor $\gamma \in (0,1]$.
\end{assumption}

\begin{theorem}[Expected Acceptance Length]
\label{thm:accept_length}
Under Assumption~\ref{ass:accept_model}, for greedy-path speculation with depth $D$:
\begin{equation}
\mathbb{E}[\tau] = \sum_{l=1}^{D} p^l \gamma^{l(l-1)/2}.
\label{eq:expect_tau}
\end{equation}
\end{theorem}

\begin{theorem}[Optimal Depth]
\label{thm:optimal_depth}
Under Assumption~\ref{ass:accept_model} with $\gamma = 1$ and width $W_l = 1$, the optimal depth satisfies:
\begin{equation}
D^* = \left\lfloor \frac{\log(c_t / c_d)}{\log(1/p)}- 1 \right\rfloor ,
\label{eq:optimal_depth}
\end{equation}
where $c_d, c_t$ are draft and target model costs. $D^*$ increases monotonically with $p$.
\end{theorem}

\begin{theorem}[Optimal Width]
\label{thm:optimal_width}
For a depth-1 tree with acceptance probability $q_i = q_1/i^\beta$, the optimal width scales as:
\begin{equation}
W^* \approx \left(\frac{q_1}{c_d}\right)^{1/(1+\beta)}.
\label{eq:optimal_width}
\end{equation}
The relative benefit of large $W$ decreases as $q_1$ increases.
\end{theorem}

\paragraph{Implications for Adaptive Strategy.} 
These theorems justify our adaptive tree construction:
\begin{itemize}[leftmargin=*,nosep]
    \item \textbf{High confidence ($\alpha \to 1$):} High acceptance probability $p$ leads to large optimal depth $D^*$, while the marginal benefit of width decreases.
    \item \textbf{Low confidence ($\alpha \to 0$):} Low $p$ makes deep speculation inefficient, but wider exploration at shallow depths increases acceptance chances.
\end{itemize}

Our linear mappings $D(\alpha) = D_{\min} + \alpha \cdot (D_{\max} - D_{\min})$ and $W(\alpha) = W_{\min} + (1-\alpha) \cdot (W_{\max} - W_{\min})$ capture these relationships. While the theorems rely on simplified models, our empirical results in Section~\ref{S5} confirm that this strategy consistently outperforms static configurations, validating the qualitative predictions of our analysis.

\section{Experiments}
\label{S5}

\begin{table*}[!ht]
  \centering
  \scriptsize
  \begin{tabular}{c|c|ccc|ccc|ccc|ccc}
    \toprule
    \multirow{2}{*}{Task} & \multirow{2}{*}{Method} &
    \multicolumn{3}{c|}{TextVQA} &
    \multicolumn{3}{c|}{GQA} &
    \multicolumn{3}{c|}{ChartQA} &
    \multicolumn{3}{c}{SEED-Bench} \\
    
    \cmidrule{3-5} \cmidrule{6-8} \cmidrule{9-11} 
    \cmidrule{12-14} 
    & & $\tau$ & Tokens/s & Speedup
    & $\tau$ & Tokens/s & Speedup 
    & $\tau$ & Tokens/s & Speedup 
    & $\tau$ & Tokens/s & Speedup \\
    \midrule
    \multirow{5}{*}{\makecell{Image}}

    & Vanilla AR & - &7.25  & - & - &9.10  & - & - &7.40  & - & - &6.79  & - \\
    & SD-Chain  &4.01  &15.86  & 2.19$\times$ &3.84  &18.15  & 1.99$\times$ &4.21  &16.60  & 2.24$\times$ &4.18  &15.30  & 2.25$\times$ \\
    & SD-Tree &3.87  &15.01  &2.07$\times$ &4.05  &19.39  & 2.13$\times$ &4.05  &15.82  & 2.14$\times$ &3.80  &14.01  & 2.06$\times$ \\
    & SpecVLM &3.93 &16.48   & 2.27$\times$ &4.01  &19.56  & 2.15$\times$ &3.84  &16.42  & 2.22$\times$ &3.84  &15.97  & 2.35$\times$ \\
    & SAGE\cellcolor{gray!20} &\textbf{5.42}\cellcolor{gray!20}  &\textbf{18.55}\cellcolor{gray!20}  & \textbf{2.56$\times$}\cellcolor{gray!20} &\cellcolor{gray!20}\textbf{5.26}  &\cellcolor{gray!20}\textbf{20.23}  &\cellcolor{gray!20}\textbf{2.22$\times$} &\cellcolor{gray!20}\textbf{5.60} &\cellcolor{gray!20}\textbf{18.95}  &\cellcolor{gray!20}\textbf{2.56$\times$} &\cellcolor{gray!20}\textbf{5.78} &\cellcolor{gray!20}\textbf{19.08}  &\cellcolor{gray!20}\textbf{2.81$\times$} \\
    \midrule
   \multirow{2}{*}{Task} & \multirow{2}{*}{Method} &
    \multicolumn{3}{c|}{VideoDetailedCaption} &
    \multicolumn{3}{c|}{MVBench} &
    \multicolumn{3}{c|}{MVLU} &
    \multicolumn{3}{c}{LongVideoBench} \\
    
    \cmidrule{3-5} \cmidrule{6-8} \cmidrule{9-11} 
    \cmidrule{12-14} 
    & & $\tau$ & Tokens/s & Speedup
    & $\tau$ & Tokens/s & Speedup 
    & $\tau$ & Tokens/s & Speedup 
    & $\tau$ & Tokens/s & Speedup \\
    \midrule
    \multirow{4}{*}{\makecell{ Video}} 
    & Vanilla AR & - & 4.69 & - & - &4.68  & - & - &4.75  & - & - &4.69  & - \\
        & SD-Chain  &3.10  &8.26  & 1.76$\times$ &3.29  &8.69  & 1.86$\times$ &3.97  &9.99  & 2.10$\times$ &3.60  &8.70  & 1.86$\times$ \\
    & SD-Tree &4.21  &10.51  &2.24$\times$ &3.50  &9.19  & 1.96$\times$ &4.22  &10.70  & 2.25$\times$ &\textbf{3.91}  &9.90  & 2.11$\times$ \\
    & SpecVLM &4.11 &13.62   &2.90$\times$ &3.45  &12.24  & 2.62$\times$ &4.07  &14.07  & 2.96$\times$ &3.16  &11.13  & 2.37$\times$ \\
    & SAGE\cellcolor{gray!20} &\cellcolor{gray!20}\textbf{5.74}  &\cellcolor{gray!20}\textbf{15.75}  &\cellcolor{gray!20}\textbf{3.36$\times$} &\cellcolor{gray!20}\textbf{5.94} &\cellcolor{gray!20}\textbf{16.12} &\cellcolor{gray!20}\textbf{3.44$\times$} &\cellcolor{gray!20}\textbf{5.19}  &\cellcolor{gray!20}\textbf{15.68}  &\cellcolor{gray!20}\textbf{3.30$\times$} &\cellcolor{gray!20}3.81  &\cellcolor{gray!20}\textbf{12.55}  &\cellcolor{gray!20}\textbf{2.68$\times$} \\
    \bottomrule
  \end{tabular}
  \caption{Average accepted length $\tau$, decoding speed (tokens/s), and speedup of LLaVA-OneVision series (72B–8B) on image tasks TextVQA, GQA, ChartQA, and SEED-Bench, and on video tasks VideoDetailedCaption, MVBench, MVLU, and LongVideoBench. ``Vanilla AR'' refers to vanilla auto-regressive decoding, ``SD-Chain'' denotes speculative decoding with draft chains, and  ``SD-Tree'' denotes speculative decoding with draft trees.}
  \label{table_1}
\end{table*}

\begin{table*}[!ht]
  \centering
  \scriptsize
  \begin{tabular}{c|c|ccc|ccc|ccc}
    \toprule
    \multirow{2}{*}{Setup} & \multirow{2}{*}{Method} &
    \multicolumn{3}{c|}{VideoDetailedCaption} &
    \multicolumn{3}{c|}{MVBench}  &
    \multicolumn{3}{c}{LongVideoBench} \\
    
    \cmidrule{3-5} \cmidrule{6-8} \cmidrule{9-11} 
    & & $\tau$ & Tokens/s & Speedup
    & $\tau$ & Tokens/s & Speedup 
    & $\tau$ & Tokens/s & Speedup 
     \\
    \midrule
    \multirow{4}{*}{\makecell{Qwen2.5-VL\\72B-7B}} 
    & Vanilla AR & - &4.99  & - & - &5.74  & - & - &4.75  & -  \\
    & SD-Chain  &3.67  &10.38  & 2.08$\times$ &3.35  &11.19  & 1.95$\times$ &3.10  &8.67  & 1.83$\times$ \\
    & SD-Tree  &4.18  &11.28  & 2.26$\times$ &3.92  &12.36  & 2.15$\times$ &\textbf{3.70}  &9.72  & 2.05$\times$  \\
    & SpecVLM  &3.90  &13.48  & 2.70$\times$ &3.73  &14.05  & 2.45$\times$ &3.27  &11.51  & 2.42$\times$    \\
    & SAGE\cellcolor{gray!20}  &\cellcolor{gray!20}\textbf{5.18}  &\cellcolor{gray!20}\textbf{15.87}  &\cellcolor{gray!20}\textbf{3.18$\times$} &\cellcolor{gray!20}\textbf{4.53}  &\cellcolor{gray!20}\textbf{15.52}  &\cellcolor{gray!20}\textbf{2.70$\times$} &\cellcolor{gray!20}3.51  &\cellcolor{gray!20}\textbf{11.97}  &\cellcolor{gray!20}\textbf{2.52$\times$}  \\
    \midrule
  \end{tabular}
  \caption{Average accepted length $\tau$, decoding speed (tokens/s), and speedup of Qwen2.5-VL series on VideoDetailedCaption, MVBench, MVLU, and LongVideoBench. ``Vanilla AR'' refers to vanilla auto-regressive decoding, ``SD-Chain'' denotes speculative decoding with draft chains, and  ``SD-Tree'' denotes speculative decoding with draft trees. }
  \label{table_2}
\end{table*}

\subsection{Experimental Setup}
\noindent\textbf{Benchmarks and Models.} To verify the effectiveness of our method in accelerating the decoding phase, we evaluate it on both image and video understanding tasks that require generating long texts. For image understanding tasks, we use the TextVQA~\cite{singh2019towards}, GQA~\cite{hudson2019gqa}, ChartQA~\cite{masry2022chartqa} and SEED-Bench~\cite{li2023seed} dataset as the benchmark. For video understanding tasks, we select VideoDetailCaption~\cite{ji2025specvlm}, MVBench~\cite{li2024mvbench}, MVLU~\cite{zhou2024mlvu}, and LongVideoBench~\cite{wu2024longvideobench} as benchmarks. We evaluate our method across various open-source VLM families, including LLaVA-OneVision series~\cite{li2024llava}, Qwen2.5-VL series~\cite{bai2025qwen2} and Qwen3-VL series~\cite{Qwen3-VL}.

\noindent\textbf{Baselines.} We compare SAGE against four methods: (1) Vanilla AR, 
the standard auto-regressive decoding serving as the speedup baseline; (2) SD-Chain, 
classic speculative decoding with a fixed-length linear draft sequence; (3) SD-Tree, 
tree-based speculative decoding with a static predefined structure; and (4) SpecVLM~\citep{ji2025specvlm}, 
a VLM-specific framework that combines visual token pruning with static tree drafting. 

\noindent\textbf{Evaluation Metrics.} Since speculative decoding methods are all lossless, we focus on acceleration metrics. We select the following metrics to evaluate acceleration performance: (1) decoding speed, (2) speed-up ratio relative to vanilla auto-regressive decoding, and (3) average accepted length.

\noindent\textbf{Implementation details.}  For the adaptive tree configuration, we set $D_{\max}=8$, $D_{\min}=3$, $W_{\max}=10$, $W_{\min}=2$, $k=10$, and the maximum node count to 64. For adaptive history tracking, we use a sliding window of size 10 with lower and upper thresholds of 2 and 3 respectively. Following SpecVLM, we apply visual token pruning in the draft model with a pruning ratio of 90\% for video tasks and 80\% for image tasks. We use greedy decoding throughout. All experiments are conducted on a server equipped with NVIDIA H20 GPUs.

\subsection{Performance Evaluation on Dense Models}
Table~\ref{table_1} presents the performance comparison on the LLaVA-OneVision series with a 72B target model and an 8B draft model across four image benchmarks and four video benchmarks. On image tasks, SAGE consistently outperforms all baselines across all metrics. Specifically, SAGE achieves average acceptance lengths of 5.42, 5.26, 5.60, and 5.78 on TextVQA, GQA, ChartQA, and SEED-Bench respectively, representing improvements of 37.9\%, 31.2\%, 45.8\%, and 50.5\% over SpecVLM. The corresponding speedup ratios reach 2.56$\times$, 2.22$\times$, 2.56$\times$, and 2.81$\times$ over vanilla auto-regressive decoding. On video tasks, the improvements are even more pronounced. SAGE achieves speedup ratios of 3.36$\times$, 3.44$\times$, 3.30$\times$, and 2.68$\times$ on VideoDetailedCaption, MVBench, MVLU, and LongVideoBench respectively, substantially exceeding both SD-Tree and SpecVLM. Notably, on MVBench, SAGE attains an acceptance length of 5.94, which is 72.2\% higher than SpecVLM's 3.45, demonstrating the effectiveness of our entropy-guided adaptive tree construction in capturing prediction confidence and optimizing speculation depth accordingly.

To validate the generalizability of our approach, we further evaluate SAGE on the Qwen2.5-VL series with a 72B target model and a 7B draft model as shown in Table~\ref{table_2}. SAGE consistently achieves the best performance across all three video benchmarks. On VideoDetailedCaption, SAGE obtains an acceptance length of 5.18 and a speedup of 3.18$\times$, outperforming SpecVLM by 32.8\% in acceptance length and 17.8\% in speedup ratio. On MVBench, SAGE achieves 4.53 average acceptance length with 2.70$\times$ speedup, compared to 3.73 and 2.45$\times$ for SpecVLM. These results demonstrate that our entropy-guided adaptive strategy generalizes effectively across different VLM architectures, consistently improving both acceptance length and inference throughput without requiring any architecture-specific modifications or additional training.

\subsection{Performance Evaluation on MoE Models}
To investigate the applicability of our approach to Mixture-of-Experts~(MoE) architectures, we conduct experiments on Qwen3-VL with 235B total parameters and 22B activated, using an 8B dense draft model. As shown in Table~\ref{table:qwen3vl}, SAGE achieves an acceptance length of 4.39 and a speedup of 1.32$\times$ on VideoDetailedCaption, outperforming both SD-Tree with 3.58 acceptance length and 1.22$\times$ speedup, and SpecVLM with 3.24 acceptance length and 1.14$\times$ speedup. The relatively modest speedup compared to dense models is expected, as MoE architectures already benefit from sparse activation during inference, reducing the computational gap between draft and target models. Nevertheless, SAGE consistently delivers the highest acceptance length and throughput, demonstrating that our entropy-guided adaptive strategy remains effective for MoE-based VLMs where the draft-target alignment poses additional challenges due to architectural heterogeneity.

\begin{table}[t]
  \centering
  \scriptsize
  \begin{tabular}{c|c|ccc}
    \toprule
    \multirow{2}{*}{Setup} & \multirow{2}{*}{Method} &
    \multicolumn{3}{c}{VideoDetailedCaption} \\
    \cmidrule{3-5}
    & & $\tau$ & Tokens/s & Speedup \\
    \midrule
    \multirow{4}{*}{\makecell{Qwen3-VL\\235B/A22B-8B}}
    & Vanilla AR & -- & 3.27 & -- \\
    & SD-Tree & 3.58 & 3.99 & 1.22$\times$ \\
    & SpecVLM & 3.24 & 3.74 & 1.14$\times$ \\
    & SAGE\cellcolor{gray!20} & \cellcolor{gray!20}4.39 & \cellcolor{gray!20}4.33 & \cellcolor{gray!20}1.32$\times$ \\
    \bottomrule
  \end{tabular}
  \caption{Performance evaluation on MoE architecture. Average accepted length $\tau$, decoding speed (tokens/s), and speedup of Qwen3-VL (235B total parameters with 22B activated, using 8B draft model) on VideoDetailedCaption.}
  \label{table:qwen3vl}
\end{table}

\begin{table}[t]
  \centering
  \scriptsize
  \begin{tabular}{c|c|ccc}
    \toprule
    \multirow{2}{*}{Dataset} & \multirow{2}{*}{Method} &
    \multicolumn{3}{c}{Llama3\,8B-1B} \\
    \cmidrule{3-5}
    & & $\tau$ & Tokens/s & Speedup \\
    \midrule
    \multirow{3}{*}{Gsm8k}
    & Vanilla & -- & 21.01 & -- \\
    & Native-SD & 3.12 & 44.59 & 2.12$\times$ \\
    & SAGE\cellcolor{gray!20} & 4.05\cellcolor{gray!20} & 48.31\cellcolor{gray!20} & 2.30$\times$\cellcolor{gray!20} \\
    \midrule
    \multirow{3}{*}{Humaneval}
    & Vanilla & -- & 18.80 & -- \\
    & Native-SD & 3.48 & 44.89 & 2.39$\times$ \\
    & SAGE\cellcolor{gray!20} & 4.89\cellcolor{gray!20} & 50.20\cellcolor{gray!20} & 2.67$\times$\cellcolor{gray!20} \\
    \bottomrule
  \end{tabular}
  \caption{Performance evaluation on large language model tasks using Llama3 (8B target model with 1B draft model). SAGE consistently outperforms native speculative decoding on both mathematical reasoning (Gsm8k) and code generation (Humaneval) benchmarks.}
  \label{table:merged}
\end{table}

\subsection{Performance Evaluation on LLMs}
To demonstrate that our entropy-guided adaptive strategy is not limited to vision-language models, we evaluate SAGE on pure language generation tasks using Llama3 with an 8B target model and a 1B draft model. As shown in Table~\ref{table:merged}, SAGE consistently outperforms native speculative decoding on both mathematical reasoning~Gsm8k and code generation~Humaneval benchmarks. On Gsm8k, SAGE achieves an acceptance length of 4.05 compared to 3.12 for Native-SD, yielding a speedup of 2.30$\times$. On Humaneval, SAGE attains an acceptance length of 4.89, which is 40.5\% higher than Native-SD, with a speedup of 2.67$\times$. These results validate the generality of our approach beyond VLMs.

\subsection{Ablation Studies}

\begin{table}[t]
  \centering
  \scriptsize
  \begin{tabular}{c|c|ccc}
    \toprule
    \multirow{2}{*}{Setup} & \multirow{2}{*}{Method} &
    \multicolumn{3}{c}{VideoDetailedCaption} \\
    
    \cmidrule{3-5} 
    & & $\tau$ & Tokens/s & Speedup 
     \\
    \midrule
    \multirow{8}{*}{\makecell{LLaVA-OV\\72B-7B}} 
    & Vanilla AR & - &4.69  & -  \\
    & $\text{SAGE}_0$  &5.60  &11.11  & 2.37$\times$  \\
    & $\text{SAGE}_{0.3}$  &5.70  &12.66  & 2.70$\times$  \\
    & $\text{SAGE}_{0.6}$  &5.81  &14.67  & 3.13$\times$  \\
    & $\text{SAGE}_{0.7}$  &5.82  &15.26  & 3.25$\times$  \\
     & $\text{SAGE}_{0.8}$  &5.83  &15.73  & 3.35$\times$   \\
     & $\text{SAGE}_{0.9}$  &5.74  &15.75  & 3.36$\times$   \\
     & $\text{SAGE}_{0.95}$  &5.55  &15.16  & 3.23$\times$   \\
    \bottomrule
  \end{tabular}
  \caption{Ablation study on the effect of visual token pruning ratio. Results are reported on VideoDetailedCaption using LLaVA-OneVision (72B-7B). The subscript denotes the proportion of visual tokens pruned. }
  \label{table_3}
\end{table}

\noindent\textbf{Effect of Pruning Ratio.} We investigate the impact of visual token pruning ratio on SAGE's performance using LLaVA-OneVision 72B-8B on VideoDetailedCaption. As shown in Table~\ref{table_3}, without pruning, $\text{SAGE}_0$ achieves an acceptance length of 5.60 but only 2.37$\times$ speedup due to the overhead of processing all visual tokens. As the pruning ratio increases, both throughput and speedup improve while the acceptance length remains stable. The optimal performance is achieved at pruning ratios between 0.8 and 0.9, where $\text{SAGE}_{0.9}$ attains 5.74 acceptance length with 3.36$\times$ speedup. However, excessive pruning with $\text{SAGE}_{0.95}$ leads to degradation as aggressive pruning removes informative visual tokens. These results suggest that moderate pruning effectively reduces draft model overhead while preserving sufficient visual context for accurate speculation.

\begin{figure}[t]
\centering
\includegraphics[width=0.472\textwidth]{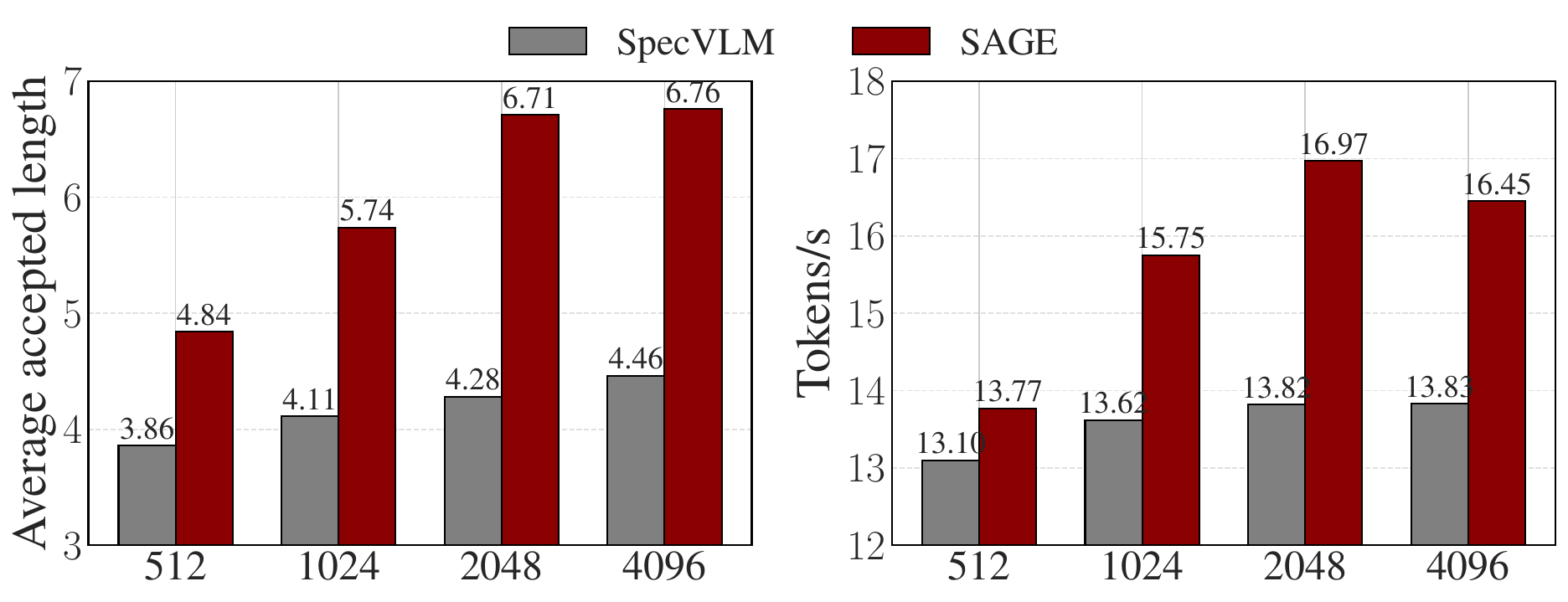}
\caption{ Performance comparison between SAGE and SpecVLM across 
different generation lengths on VideoDetailedCaption (LLaVA-OneVision 72B-7B).}
\label{fig2}
\end{figure}

\noindent\textbf{Effect of Generation Token Length.} As shown in Figure~\ref{fig2}, we analyze how generation length affects performance by comparing SAGE and SpecVLM. Both methods exhibit improved acceptance lengths as generation length increases, attributed to more predictable generation patterns in longer sequences. However, SAGE consistently outperforms SpecVLM across all settings, with the performance gap widening as generation length grows. At 512 tokens, SAGE achieves 25.4\% improvement over SpecVLM in acceptance length. This gap increases substantially at longer generations: 39.7\% improvement at 1024 tokens and 56.8\% at 2048 tokens. The throughput advantage follows a similar trend. These results demonstrate that our entropy-guided adaptive strategy becomes increasingly effective for longer generation tasks, as the method can better exploit the temporal patterns of prediction confidence over extended sequences.

\section{Related Work}

\subsection{Efficient Inference for Vision-Language Models}

Vision-language models such as BLIP-2~\cite{li2023blip}, LLaVA~\cite{lin2024video}, and Qwen-VL~\cite{Qwen3-VL} achieve remarkable multimodal performance but suffer from substantial computational demands. Various acceleration techniques have been proposed, including token pruning~\cite{ye2025fit} and quantization~\cite{Tong_2025_ICCV,10.1145/3771936}. SpecVLM~\cite{ji2025specvlm} introduces visual token pruning in the draft model to bridge the modality gap for speculative decoding in VLMs. Our work takes an orthogonal perspective by dynamically adapting the tree structure based on prediction uncertainty, which can be integrated with existing techniques.

\subsection{Speculative Decoding}
Speculative decoding~\cite{leviathan2023fast,sun2023spectr} accelerates auto-regressive generation by using a smaller draft model to generate candidate tokens verified in parallel by the target model. To improve acceptance rates, tree-structured drafting has been widely adopted. SpecInfer~\cite{miao2024specinfer} organizes candidates as a tree for simultaneous verification. Medusa~\cite{cai2024medusa} adds multiple decoding heads for tree-structured generation, while EAGLE~\cite{li2024eagle} employs feature-level auto-regressive drafting. However, most speculative decoding research focuses on text-only LLMs. The unique challenges of VLMs—including the modality gap and longer context lengths due to visual tokens—have received less attention. Only a few works such as SpecVLM~\cite{ji2025specvlm} have explored this direction, yet they still rely on static tree configurations, which motivates our adaptive approach.

\subsection{Adaptive Decoding Strategies}
Adaptive methods that adjust decoding behavior based on runtime signals have shown promise in various contexts. CALM~\cite{schuster2022confident} uses confidence for early exit in transformer layers, and entropy-based uncertainty estimation has been widely adopted for uncertainty quantification in language models~\cite{kuhn2023semantic} and model calibration~\cite{guo2017calibration}. In LLM speculative decoding, EAGLE-2~\cite{li2024eagle2} introduced context-aware dynamic draft trees by using token-level confidence scores to guide branch expansion during drafting.  Our SAGE differs in three aspects: (1) we use \textit{distributional entropy} rather than individual token probabilities as the confidence measure; (2) we adopt a \textit{lookahead} strategy that pre-plans tree structure for the next round based on current entropy, exploiting temporal correlation across decoding steps; (3) we perform \textit{global resource allocation} by jointly modulating depth and width inversely. To our knowledge, SAGE is the first entropy-guided adaptive tree construction method for speculative decoding in VLMs.

\section{Conclusions}

We introduce SAGE, a novel framework that dynamically adjusts the draft tree structure based on real-time prediction uncertainty for vision-language models. Our key insight is that output entropy serves as a natural confidence indicator: low entropy enables deeper speculation, while high entropy favors wider exploration at shallow depths. By adaptively constructing deeper-narrower trees for confident predictions and shallower-wider trees for uncertain ones, SAGE achieves improved acceptance lengths and faster acceleration  compared to static tree baselines.

\bibliography{example_paper}
\bibliographystyle{icml2025}

\newpage
\appendix
\onecolumn
\section{Appendix}

\subsection{Initialized draft tree structure}
Following SpecVLM~\cite{ji2025specvlm}, we adopt the tree structure illustrated in Figure~\ref{tree} for both static baselines and SAGE's initial configuration. For SD-Tree and SpecVLM, this structure remains fixed during decoding. SAGE uses it only as initialization, then dynamically adjusts tree depth and width based on prediction entropy after each iteration.

\begin{figure*}[th!]
\centering
\includegraphics[width=0.9\textwidth]{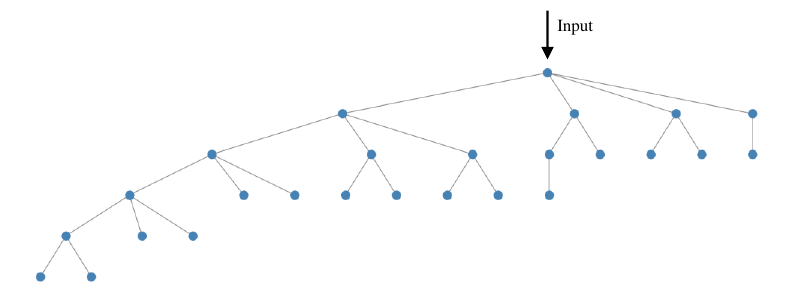}
\caption{Initialized draft tree structure.}
\label{tree}
\end{figure*}

\subsection{Analysis of Computational Time Overhead}
\label{sec:time_analysis}
To understand the efficiency of SAGE, we conduct a detailed analysis of the computational time distribution across different operations during inference. Table~\ref{tab6} presents the time proportion of each operation. The decoding phase dominates the total inference time, with target model verification and draft model decoding accounting for 36.59\% and 34.81\% respectively, which aligns with the motivation of speculative decoding that aims to accelerate auto-regressive generation through parallel verification. The target model prefilling consumes 26.72\% of the total time due to the processing of visual tokens from video frames, while the draft model prefilling only requires 0.77\%. Most importantly, the additional overhead introduced by SAGE is negligible: the entropy-based dynamic tree update requires only 0.39\% of the total time. This minimal overhead validates our design choice of using output entropy as a lightweight confidence indicator, as entropy computation leverages the probability distribution already available from the softmax operation without requiring additional forward passes.

\begin{table}[htpb]
\centering
\scriptsize
\begin{tabular}{c|ccccccccc}
\toprule
Operation         &Target Prefilling  & Target Verification  &Draft Prefilling  &Draft Decoding &Token Pruning  &Tree Update &Others &Total   \\ \midrule
Time Proportion          &26.72\%    &36.59\%    &0.77\%   &34.81\%    &0.03\%    &0.39\%    &0.69\%   &100\%      \\
\midrule
\end{tabular}
\caption{Computational time breakdown of SAGE during inference.}
\label{tab6}
\end{table}

\subsection{Detailed Theoretical Proofs}
\label{app:proofs}

\begin{proof}[Proof of Lemma~\ref{lem:conf_prob}]
We solve the optimization problem:
\begin{equation}
\min_{p_1, \ldots, p_k} p_1 \quad \text{s.t.} \quad H(\tilde{P}_d) = (1-\alpha)\log k, \quad \sum_i p_i = 1, \quad p_i \geq 0, \quad p_1 \geq p_2 \geq \cdots \geq p_k.
\end{equation}

For fixed $p_1$, entropy is maximized when the remaining mass $(1-p_1)$ is distributed uniformly among the other $k-1$ tokens, giving $p_i = (1-p_1)/(k-1)$ for $i \geq 2$. The resulting entropy is:
\begin{equation}
H^*(p_1) = -p_1 \log p_1 - (1-p_1) \log \frac{1-p_1}{k-1}.
\end{equation}

Since $H^*(p_1)$ is strictly decreasing in $p_1$ for $p_1 \in [1/k, 1]$, and $H^*(1/k) = \log k$ while $H^*(1) = 0$, there exists a unique $p_1^*$ satisfying $H^*(p_1^*) = (1-\alpha)\log k$.

For the constraint $H(\tilde{P}_d) = (1-\alpha)\log k$ to be achievable with $p_1$ as small as possible, we need $H^*(p_1) \geq (1-\alpha)\log k$. Using the concavity of entropy and linear interpolation between the boundary cases:
\begin{equation}
p_1 \geq \frac{1}{k} + \alpha \cdot \left(1 - \frac{1}{k}\right) = \frac{1}{k} + \alpha \cdot \frac{k-1}{k}.
\end{equation}
This bound is tight when $\alpha \in \{0, 1\}$.
\end{proof}

\begin{proof}[Proof of Theorem~\ref{thm:accept_prob}]
By the definition of total variation distance, for any token $y$:
\begin{equation}
|P_d(y|c) - P_t(y|c)| \leq 2D_{TV}(P_d(\cdot|c), P_t(\cdot|c)) \leq 2\epsilon.
\end{equation}
Thus $P_t(\hat{y}|c) \geq P_d(\hat{y}|c) - 2\epsilon = p_1 - 2\epsilon$.

For $\hat{y}$ to be the target model's argmax, we need $P_t(\hat{y}|c) > P_t(y|c)$ for all $y \neq \hat{y}$. In the worst case, the target model places all non-$\hat{y}$ probability mass on a single alternative token $y'$:
\begin{equation}
P_t(y'|c) \leq 1 - P_t(\hat{y}|c) \leq 1 - (p_1 - 2\epsilon).
\end{equation}

The acceptance condition $P_t(\hat{y}|c) > P_t(y'|c)$ is satisfied when:
\begin{equation}
p_1 - 2\epsilon > 1 - (p_1 - 2\epsilon) \implies p_1 > \frac{1}{2} + 2\epsilon.
\end{equation}

Substituting the lower bound from Lemma~\ref{lem:conf_prob}:
\begin{equation}
\frac{1}{k} + \alpha \cdot \frac{k-1}{k} > \frac{1}{2} + 2\epsilon.
\end{equation}

Solving for $\alpha$:
\begin{equation}
\alpha > \frac{k}{k-1}\left(\frac{1}{2} + 2\epsilon - \frac{1}{k}\right) = \frac{k - 2 + 4\epsilon k}{2(k-1)}.
\end{equation}
\end{proof}

\begin{remark}
The bound in Theorem~\ref{thm:accept_prob} is conservative, providing a \emph{sufficient} condition for guaranteed acceptance. In practice, acceptance can occur even when $p_1 \leq 1/2 + \epsilon$ if the target model's probability mass is not adversarially distributed. Our empirical results show that acceptance probability increases smoothly with $\alpha$, suggesting the actual relationship is stronger than this worst-case bound.
\end{remark}

\begin{corollary}[Multi-Step Acceptance]
\label{cor:multi_step}
For a path of depth $d$ in the speculation tree, assuming independence across steps and that each step has confidence $\alpha_l$ satisfying the threshold in Eq.~\ref{eq:alpha_threshold}, all $d$ tokens are guaranteed to be accepted. For confidence levels below the threshold, the acceptance probability decreases with depth.
\end{corollary}

\begin{proof}[Proof of Theorem~\ref{thm:accept_length}]
Let $A_l$ denote the event that the token at depth $l$ is accepted. The probability of accepting exactly $l$ tokens is:
\begin{equation}
P(\tau = l) = P(A_1, \ldots, A_l, \overline{A_{l+1}}) = \left(\prod_{j=1}^{l} p_j\right)(1 - p_{l+1}), \quad \text{for } l < D,
\end{equation}
and $P(\tau = D) = \prod_{j=1}^{D} p_j$.

The expected acceptance length is:
\begin{align}
\mathbb{E}[\tau] &= \sum_{l=1}^{D-1} l \cdot \left(\prod_{j=1}^{l} p_j\right)(1 - p_{l+1}) + D \cdot \prod_{j=1}^{D} p_j \\
&= \sum_{l=1}^{D-1} l \cdot \prod_{j=1}^{l} p_j - \sum_{l=1}^{D-1} l \cdot \prod_{j=1}^{l+1} p_j + D \cdot \prod_{j=1}^{D} p_j \\
&= \sum_{l=1}^{D-1} l \cdot \prod_{j=1}^{l} p_j - \sum_{l=2}^{D} (l-1) \cdot \prod_{j=1}^{l} p_j + D \cdot \prod_{j=1}^{D} p_j \\
&= p_1 + \sum_{l=2}^{D-1} \prod_{j=1}^{l} p_j + \prod_{j=1}^{D} p_j = \sum_{l=1}^{D} \prod_{j=1}^{l} p_j.
\end{align}

Substituting $p_j = p \cdot \gamma^{j-1}$:
\begin{equation}
\prod_{j=1}^{l} p_j = p^l \cdot \gamma^{0+1+\cdots+(l-1)} = p^l \cdot \gamma^{l(l-1)/2}.
\end{equation}
\end{proof}

\begin{proof}[Proof of Theorem~\ref{thm:optimal_depth}]
With $\gamma = 1$ and $W_l = 1$, the tree has $|\mathcal{T}| = D$ nodes, and from Theorem~\ref{thm:accept_length}:
\begin{equation}
\mathbb{E}[\tau] = \sum_{l=1}^{D} p^l = p \cdot \frac{1 - p^D}{1 - p}.
\end{equation}

The speedup ratio becomes:
\begin{equation}
S(D) = \frac{p \cdot \frac{1 - p^D}{1 - p} + 1}{c_d \cdot D + c_t}.
\end{equation}

The marginal benefit of increasing depth from $D$ to $D+1$ is:
\begin{equation}
\Delta_{\text{benefit}}(D) = p^{D+1}.
\end{equation}

The marginal benefit of increasing depth is $p^{D+1}$ (additional expected accepted tokens), while the marginal cost is $c_d$. To compare these quantities in consistent units relative to a single target model forward pass (cost $c_t$), we require the benefit-to-cost ratio:
\begin{equation}
p^{D+1} > \frac{c_d}{c_t} \implies (D+1) \log p > \log \frac{c_d}{c_t} \implies D < \frac{\log(c_t/c_d)}{\log(1/p)} - 1.
\end{equation}

Thus the optimal depth is $D^* = \lfloor \log(c_t/c_d) / \log(1/p) \rfloor$. As $p$ increases, $\log(1/p)$ decreases, so $D^*$ increases.
\end{proof}

\begin{proof}[Proof of Theorem~\ref{thm:optimal_width}]
With depth $D=1$ and width $W$, the probability that the $i$-th candidate is accepted is $q_i = q_1/i^\beta$. The expected acceptance is upper-bounded by:
\begin{equation}
\mathbb{E}[\text{accept}] \leq \sum_{i=1}^{W} q_i = q_1 \sum_{i=1}^{W} \frac{1}{i^\beta} \approx q_1 \cdot \frac{W^{1-\beta}}{1-\beta} \quad \text{for } \beta < 1.
\end{equation}

The speedup ratio is:
\begin{equation}
S(W) = \frac{\mathbb{E}[\text{accept}] + 1}{c_d \cdot W + c_t}.
\end{equation}

Taking the derivative and setting to zero, after simplification we obtain:
\begin{equation}
q_1 W^{-\beta}(c_d W + c_t) = c_d \left( \frac{q_1 W^{1-\beta}}{1-\beta} + 1 \right).
\end{equation}

For the regime where the ``$+1$'' term dominates (low acceptance), this simplifies to:
\begin{equation}
q_1 c_d W^{1-\beta} \approx c_d \implies W^* \approx \left(\frac{1}{q_1}\right)^{1/(1-\beta)}.
\end{equation}

More generally, dimensional analysis suggests $W^* \sim (q_1/c_d)^{1/(1+\beta)}$. The key qualitative insight is that when $q_1$ is high, the first candidate is likely correct, reducing the marginal benefit of additional width. 
\end{proof}

\begin{remark}
Our theoretical analysis makes several simplifying assumptions, including independence across decoding steps and specific functional forms for acceptance probability. The primary contribution of this analysis is to establish that the \emph{direction} of our adaptations---deeper trees for confident predictions, wider trees for uncertain ones---has a principled foundation. The consistent empirical improvements across diverse settings (Section~\ref{S5}) suggest these design principles are robust even when the simplified assumptions do not hold exactly.
\end{remark}

\end{document}